\documentclass{article}

\usepackage{spconf,amsmath,graphicx,epstopdf,amssymb,algorithm,algorithmic,float}
\usepackage{balance,amsthm,subcaption}

\newcommand\myeq{\mathrel{\overset{\makebox[0pt]{\mbox{\normalfont\tiny\sffamily (a)}}}{=}}}
\newcommand\mydots{\hbox to 1em{.\hss.\hss.}}
\newtheorem{lemma}{Lemma}
\newtheorem{theorem}{Theorem}

\algsetup{indent=1em}

\ninept   

\usepackage{amsmath}
\usepackage{amssymb}
\usepackage{amsthm}
\usepackage{latexsym}
\usepackage{verbatim}
\usepackage{graphicx}

\newtheorem{prop}{Proposition}

{\begin{list}               
    {$\bullet$ \hfill}{
        \setlength{\leftmargin}{\parindent}
        \setlength{\parsep}{0.04\baselineskip}
        \setlength{\itemsep}{0.5\parsep}
        \setlength{\labelwidth}{\leftmargin}
        \setlength{\labelsep}{0em}}
    }
{\end{list}}

\providecommand{\cref}[1]{Chapter~\ref{chap:#1}}

\providecommand{\norm}[1]{\left\lVert#1\right\rVert}
\providecommand{\inprod}[1]{\left\langle#1\right\rangle}
\providecommand{\set}[1]{\left\{#1\right\}}

\providecommand{\bydef}{\overset{\text{def}}{=}}

\providecommand{\rank}{\mathop{\mathrm{rank}}}

\renewcommand{\vec}[1]{\ensuremath{\mathbf{#1}}}
\providecommand{\mat}[1]{\ensuremath{\mathbf{#1}}}

 \providecommand{\mD}{\mat{D}}

\providecommand{\mN}{\mat{N}}
  
\providecommand{\mQ}{\mat{Q}} \providecommand{\mR}{\mat{R}}
  
 \providecommand{\mT}{\mat{T}}

 \providecommand{\vb}{\vec{b}}

\providecommand{\vm}{\vec{m}} \providecommand{\vn}{\vec{n}} 
 \providecommand{\vp}{\vec{p}}
\providecommand{\vq}{\vec{q}} \providecommand{\vr}{\vec{r}}
\providecommand{\vs}{\vec{s}}

\providecommand{\vx}{\vec{x}}

 \providecommand{\vv}{\vec{v}}

\usepackage{xcolor}

\title{Omnidirectional Bats, Point-to-Plane Distances, and\\ the Price of Uniqueness}

\name{Miranda Krekovi\'c$^{\,\dagger}$, Ivan Dokmani\'c$^{\,\ddagger}$, and Martin Vetterli$^{\,\dagger}$}
\address{\hspace{-4mm}
\begin{minipage}{.5\linewidth}
    \centering
    $^{\dagger}$
    School of Computer and Communication Sciences\\
    Ecole Polytechnique F\'ed\'erale de Lausanne (EPFL) \\CH-1015 Lausanne, Switzerland\\
    \{miranda.krekovic,martin.vetterli\}@epfl.ch
\end{minipage}%
\hspace{4mm}%
\begin{minipage}{.5\linewidth}
    \centering
    $^{\,\ddagger}$
    Institut Langevin\\
    CNRS, ESPCI Paris, PSL Research University\\
    1 rue Jussieu, 75005 Paris, France\\
    ivan.dokmanic@espci.fr
\end{minipage}%
}

\newcommand\blfootnote[1]{%
	\begingroup
	\renewcommand\thefootnote{}\footnote{#1}%
	\addtocounter{footnote}{-1}%
	\endgroup
}

\begin{document}
%
\maketitle
\begin{abstract}

We study simultaneous localization and mapping with a device that uses reflections to measure its distance from walls. Such a device can be realized acoustically with a synchronized collocated source and receiver; it behaves like a bat with no capacity for directional hearing or vocalizing. In this paper we generalize our previous work in 2D, and show that the 3D case is not just a simple extension, but rather a fundamentally different inverse problem. While generically the 2D problem has a unique solution, in 3D uniqueness is always absent in rooms with fewer than nine walls. In addition to the complete characterization of ambiguities which arise due to this non-uniqueness, we propose a robust solution for inexact measurements similar to analogous results for Euclidean Distance Matrices. Our theoretical results have important consequences for the design of collocated range-only SLAM systems, and we support them with an array of computer experiments.
\end{abstract}

\blfootnote{This work was support by the Swiss National Science Foundation grant number 20FP-1 151073, ``Inverse Problems regularized by Sparsity''. ID was funded by LABEX WIFI (Laboratory of Excellence within the French Program ``Investments for the Future'') under references ANR-10-LABX-24 and ANR-10-IDEX-0001-02 PSL* and by Agence Nationale de la Recherche under reference ANR-13-JS09-0001-01.}

\section{Introduction}
\label{sec:intro}


Imagine an omnidirectional bat who pilots indoors by listening to echoes of its chirps, without any idea at all about where the echoes are coming from. This unusual bat faces a conundrum which is the theme of our paper: given the distances between a set of waypoints and a set of walls in a room, can we reconstruct both the trajectory and the shape of the room? Such simultaneous recovery belongs to a class of problems famously known as simultaneous localization and mapping, or SLAM \cite{thrun}.  

Prior studies have demonstrated that multipath conveys essential information about the room geometry and that this geometry can be estimated from room impulse responses (RIRs) \cite{meissner, dokthesis, antonacci,dokmanicdaudet}. Many common setups consider multiple sources or microphone arrays \cite{ribeiro, hu}. Assuming that the microphones are static, geometric relationships between the propagation times enable us to relate the room estimation problem to the so-called inverse problem for Euclidean distance matrices (EDMs), which aims to reconstruct the points in a set from their pairwise distances \cite{liberti}. Thanks to a wide range of applications, a number of tools related to EDMs have been developed to  reconstruct the original points from noisy and incomplete distances \cite{gaffke, reza, browne}.

A different approach is to assume a single omnidirectional sound source and a single omnidirectional microphone collocated on one single device \cite{krekovic,krekovic2,peng}. There are several advantages of such a setup: First, we do not assume pre-installed fixed beacons in the room. Second, we work with first-order echoes only, which are relatively easy to measure or estimate. Third, we do not require any knowledge about the trajectory of the device or the measurement locations.

A particularity of the described setup is that the propagation times directly reveal the distances between the measurement locations and the walls, and unlike the most common setups, the problem is not a specification of the inverse problem for EDMs. Nevertheless, it is related to a similar inverse problem that has not yet been studied---given the matrix of noisy point-to-plane distances (a PPDM), reconstruct the generating points and planes.

In our previous work, we characterized this inverse problem in 2D \cite{krekovic}. In this paper we generalize our prior work in two major ways. First, we expand the study from 2D to 3D. This expansion is non-trivial---it so happens that the question of uniqueness has a fundamentally different answer in 3D. We show that new ambiguities arise in addition to the usual invariance to rigid transformations characteristic of EDMs. We identify the equivalence classes of rooms and trajectories yielding the same measurements. Our analysis is exhaustive in that we obtain a complete if and only if characterization. Finally, we propose an optimization-based estimator for the noisy case which can be efficiently computed using off-the-shelf optimization tools.

\begin{figure}
\begin{center}
  \includegraphics[width=7.1cm]{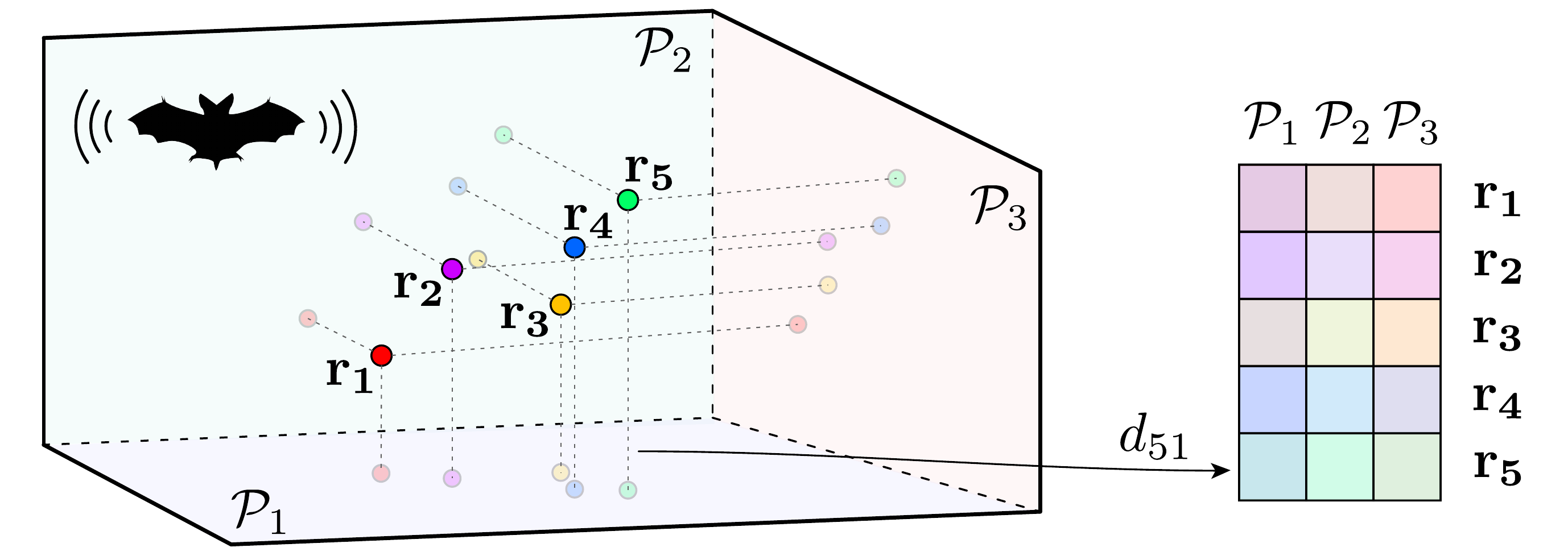}
    \caption{Illustration of $N=5$ points $\mathbf{r}_i$ and $K=3$ planes ${\cal P}_j $ with the corresponding distance matrix.}
    \vspace{-1.8em}
    \label{fig:setup}
 \end{center}
 \end{figure}

\section{Problem setup}
\label{sec:problem_setup}
We consider a scenario as in Fig. \ref{fig:setup} with $N$ waypoints $\{ {\mathbf{r}_i} \}_{i=1}^N$ and $K$ walls (planes) $\{ { {\cal P}_j } \}_{j=1}^K$.
We work with the parameterization of the plane ${\cal P}_j$ in Hessian normal form $\inprod{\vn_j, \vx } = q_j$, where $\mathbf{n}_j$ is a unit normal, $q_j = \inprod{\vn_j, \vp_j}$ is the distance of the plane from the origin and $\mathbf{p}_j$ is any point on the plane ${\cal P}_j$. 
Our measurements are the distances between the waypoints and planes:
\begin{equation}
\label{eq:distdef}
d_{i,j} = \mathrm{dist}(\vr_i, {\cal P}_j) = q_j - \langle \vr_{i},  \vn_{j} \rangle
\end{equation}
for  $i = 1, ...,N$ and $j = 1, ..., K$. We choose $N$ so that $N \geq K$, and define $\mD  \in \mathbb{R}^ {N
\times K}$ to be the \textit{points-to-planes distance matrix} (PPDM) with entries $d_{i,j}$.

\subsection{Rank of the PPDM}

Our first result is a simple proposition about the rank of the introduced PPDM, similar in spirit to the rank property for EDMs:

\begin{prop}
With $\mD$ defined as above, we have
\begin{equation}
\rank(\mD) \leq d + 1,
\end{equation}
where $d$ is the dimension of the space.
\end{prop}

\begin{proof}
  Denoting $\mN = [\vn_1, \ldots, \vn_K]$ and $\mR = [\vr_1, \ldots, \vr_N]$, we can write $\mD$ as
  \begin{equation}
    \mD = - \mR^T \mN + \vec{1} \vq^T.
  \end{equation}
  Since $\rank(\mR^T \mN)\leq d$ and
  $\rank(\vec{1} \vq^T) = 1$, the statement follows by the rank inequalities.
\end{proof}
In other words, rank of a PPDM is independent of the number of points and planes that generate it.
In real situations the measurements are unreliable: distances are noisy, it is impossible to obtain them all or they are unlabelled. That results in a noisy and incomplete matrix $\mD$.
The low-rank property (or its approximate version in the noisy case) gives us a simple heuristic for denoising $\mD$ and estimating the unobserved distances.

\section{Uniqueness of the inverse problem}

In this section we present a study of the uniqueness of the inverse problem for PPDMs. As we will see, going from 2D to 3D brings about an important change in the character of this problem.

\subsection{Invariance to Rigid Motions}

For completeness we state the following intuitive result:

\begin{lemma}
    \label{lem:rigid}
    PPDMs are invariant to rigid motions of the plane-point setup.
\end{lemma}

\begin{proof}
We represent rotations and reflections by an orthogonal matrix $\mathbf{Q} \in \mathbb{R}^{d \times d}$ and translations by the vector $\vb$ that acts on the points $\vp_j$ and $\vr_i$. We further denote by $d_{i,j}'$ the transformed distances. Then, we can write
\begin{align}
d_{i,j}' = &   \inprod{\mQ\vn_j,  \mQ (\vp_j + \vb)} - \inprod{\mQ\vn_j, \mQ(\vr_i + \vb)} \nonumber \\
\myeq & \text{ } \inprod{\vn_j,  \vp_j + \vb} - \inprod{\vn_j, \vr_i + \vb} \nonumber \\
= & \text{ }  \inprod{\vn_j, \vp_j }  - \inprod{\vn_j, \vr_i} \nonumber
= \text{ } d_{i,j}
\end{align}
where (a) follows from the orthogonality of $\mQ$.
\end{proof}

A consequence of this invariance is that the absolute position and orientation of points and planes cannot be recovered from distances only, and the corresponding degrees of freedom need to be specified separately. 

Equivalent result is known to hold for EDMs. However, in contrast to EDMs where the invariance to rigid transformations is the only one, the inverse problem of retrieving points and planes from a PPDM exhibits additional ambiguities.

\subsection{Invariances Beyond Rigid Motions}

To study the uniqueness of the inverse problem for PPDMs, we first state a result that transforms the question of uniqueness into a question about nullspace dimension of certain matrices. We will denote room-trajectory setups as pairs of planes and waypoints: ${\cal R} = (\set{{\cal P}_j}), \set{\vr_i})$, and the corresponding PPDMs as $\mD({\cal R})$.

\begin{lemma} 
Let ${\cal R}_1= (\set{{\cal P}_j}), \set{\vr_i})$ and ${\cal R}_2= (\set{{\cal Q}_j}), \set{\vs_i})$ be two room-trajectory setups with the corresponding normals being $\set{\vn_{j}}$ and $\set{\vm_{j}}$. Then $\mD({\cal R}_1) = \mD({\cal R}_2)$ if and only if
\[
\mR_0^T \mN_0 = \vec{0}
\]
where
$ \mathbf{R}_0 \bydef
\begin{bmatrix}
\vr_1  & \hdots & \vr_N  \\
-\vs_1  & \hdots & -\vs_N  \\
\end{bmatrix}, \ 
\mathbf{N}_0 \bydef \begin{bmatrix}  
\vn_{1}  & \hdots & \vn_{K} \\
\vm_{1}  & \hdots & \vm_{K}  
\end{bmatrix}. $
\label{lemma1}
\end{lemma}

\begin{proof}
Assuming that $d_{i,j}({\cal R}_1)  = d_{i,j}({\cal R}_2) $, we obtain using \eqref{eq:distdef} that
\begin{equation}
\big< \vp_j, \vn_j \big> - \big< \vr_i, \vn_j \big>  = \big< \vq_j, \vm_j \big> - \big< \vs_i, \vm_j \big> \quad \forall i, j.
\label{eq:main}
\end{equation}
Instead of studying \eqref{eq:main} directly, we assume that 
\begin{equation}
\big< \vp_j, \vn_j \big> = \big< \vq_j, \vm_j \big>,
\label{eq:side_simple}
\end{equation}
so we arrive at the simplified relation
\begin{equation}
\big< \vr_i, \vn_j \big> =  \big< \vs_i, \vm_j \big>.
\label{eq:main_simple}
\end{equation} 

To solve \eqref{eq:side_simple} and \eqref{eq:main_simple}, we first search for $\vn_j,  \vm_j, \vr_i$ and $\vs_i$ that satisfy \eqref{eq:main_simple}. Then, given a set of normals $\vn_j$ and $\vm_j$, we obtain one equation \eqref{eq:side_simple} with $2d$ unknown variables $\vp_j$ and $\vq_j$ for each $j$.
Therefore, as we can always find a solution to the linear equations \eqref{eq:side_simple}, we focus on solving \eqref{eq:main_simple}, which can be written in matrix form as
\begin{equation}
\mathbf{R}_0^T \mathbf{N}_0 = \mathbf{0}.
\label{eq:rono}
\end{equation}
The solution exists when the columns of $\mathbf{N}_0$ are in the nullspace of $\mathbf{R}_0$ and the rows of $\mathbf{R}_0$ are in the nullspace of $\mathbf{N}_0^T$.

In the general case when $ \big< \vr_i, \vn_j \big> \neq  \big< \vs_i , \vm_j \big>$ we denote the difference $\big< \vp_j, \vn_j \big>  -  \big< \vq_j, \vm_j \big> = \big< \vr_i, \vn_j \big> - \big< \vs_i , \vm_j \big> = w_{i,j}$ and obtain the equations:
\begin{align}
\big< \vr_i, \vn_j \big> = \big< \vs_i , \vm_j \big> + w_{i,j}  \label{eq:proof1_1} \\
\big< \vp_j, \vn_j \big> =  \big< \vq_j , \vm_j \big> + w_{i,j} \label{eq:proof1_2}
\end{align}
We notice that all variables in \eqref{eq:proof1_2} depend only on the plane index, so we require the same for the introduced variable: $w_{i,j} = w_j$, $\forall i$. As any real number can be written as an inner product of some vectors, we let $\mathbf{v}_j \in \mathbb{R}^d$ be any vector such that $w_j = \vv^T_j \vn_j$. Then, from \eqref{eq:proof1_1} we obtain $\big< \vr_i - \vv_j, \vn_j \big> = \big< \vs_i, \vm_j \big>$ which can be written in matrix form as
\begin{equation}
\mathbf{R}^T \mathbf{N} = \mathbf{0},
\label{eq:rn}
\end{equation}
\[ \mathbf{R} =
\begingroup
\renewcommand*{\arraystretch}{1}
\begin{bmatrix}
\vr_{1} - \vv_{1} & \hspace{-2px} \mydots & \hspace{-2px} \vr_{1} - \vv_{K} & \hspace{-2px} \mydots & \hspace{-2px} \vr_{N} - \vv_{1} & \hspace{-2px} \mydots & \hspace{-2px} \vr_{N} - \vv_{K}  \\
-\vs_{1} & \hspace{-2px} \mydots & \hspace{-2px} -\vs_{1}  & \hspace{-2px} \mydots & \hspace{-2px} -\vs_{N} & \hspace{-2px} \mydots & \hspace{-2px} -\vs_{N} 
\end{bmatrix} \endgroup
, \]
\[\mathbf{N} = \begin{bmatrix}  
\mathbf{N}_0 & \mathbf{N}_0 \hdots & \mathbf{N}_0
\end{bmatrix}. \]

We now show that the assumption \eqref{eq:side_simple} does not reduce generality and prove that \eqref{eq:rono} gives the same characterization of the uniqueness property as the general case \eqref{eq:rn}.

As mentioned, the solution of \eqref{eq:rono} exists when the rows of $\mathbf{R}_0$ are in the nullspace of $\mathbf{N}_0^T$,

\begin{equation}
\begin{bmatrix} \mathbf{r}_i  \\ -\mathbf{s}_i \end{bmatrix} \in \mathcal{N}(\mathbf{N}_\mathbf{0}^T) = \mathbf{A} \mathbf{c}_i,
\end{equation}
where $\mathbf{A} \in \mathbb{R}^{2d \times d}$ is the matrix of basis vectors spanning $\mathcal{N}(\mathbf{N}_\mathbf{0}^T)$ arranged in the columns, and $\mathbf{c}_i \in \mathbb{R}^{d}$ is the vector of coefficients. 
Analogously, for \eqref{eq:rn} to have a non-zero solution, rows of the matrix $\mathbf{R}$ must satisfy
\begin{equation}
\begin{bmatrix} \mathbf{r}_i - \mathbf{v}_j \\ -\mathbf{s}_i \end{bmatrix}  \in \mathcal{N}(\mathbf{N}^T) \myeq \mathbf{A} \mathbf{c}_{i,j},
\label{eq:generalization}
\end{equation}
where (a) follows from $\mathcal{N}(\mathbf{N}_0^T) = \mathcal{N}(\mathbf{N}^T)$. Therefore, a solution has the form
\begin{equation}
\begin{bmatrix}   \mathbf{r}_i  \\ -\mathbf{s}_i \end{bmatrix}  = \begin{bmatrix} \mathbf{v} \\ \mathbf{0} \end{bmatrix} + \mathbf{A}  \mathbf{c}_{i},
\end{equation}
where $\mathbf{v}$ is a translation vector of $\mathbf{r}_i$. But by Lemma \ref{lem:rigid}, PPDMs are invariant translations so assumption \eqref{eq:side_simple} indeed does not remove any solutions.
\end{proof}

In the following, we analyze which room-trajectory setups verify the conditions of Lemma \ref{lemma1}. 
That is, we study solutions of \eqref{eq:rono} by searching for vectors  $\vr_i$ and $\vs_i$ that live in $\mathcal{N}(\mathbf{N}_0^T)$, at the same time constraining
\begin{equation}
 \mathbf{n}_j  =
\begin{bmatrix} \sin\theta \\ \cos\theta \end{bmatrix}, 
 \mathbf{m}_j  =
\begin{bmatrix} \sin\theta'  \\ \cos\theta' \end{bmatrix},
\end{equation}
in 2D and
\begin{equation}
\mathbf{n}_j  =
\begin{bmatrix} \sin\theta \cos\phi \\ \sin\theta \sin\phi \\ \cos\theta \end{bmatrix}, 
 \mathbf{m}_j  =
\begin{bmatrix} \sin\theta' \cos\phi' \\ \sin\theta' \sin\phi' \\ \cos\theta' \end{bmatrix},
\label{eq:normals}
\end{equation}
in 3D to ensure that normals remain unit vectors. Generically, for $K \geq 2d$, the nullspace is empty. To make it nonempty, we must explicitly assume linear dependencies among the columns or rows.

The analysis differs for 2D and 3D spaces and we conduct it separately. Since the analysis is straightforward but cumbersome, in consideration of limited space we prefer to present a sketch, give intuitions and illustrate rooms and trajectories that lead to same measurements, and to defer the details to a forthcoming journal version of this paper \cite{krekovic_journal}.

\vspace{-0.5em}
\subsubsection{2D space}

We distinguish two different cases that lead to valid rooms. First, we assume that the affine dimension of points $\mathbf{r}_i$ and $\mathbf{s}_i$ is lower than the ambient dimension, so that the measurement locations are collinear. Then, for any room we can find a set of equivalent rooms with respect to distance measurements. An example is given as follows:  

\vspace{0.75em}
\centerline{\includegraphics[width=7cm]{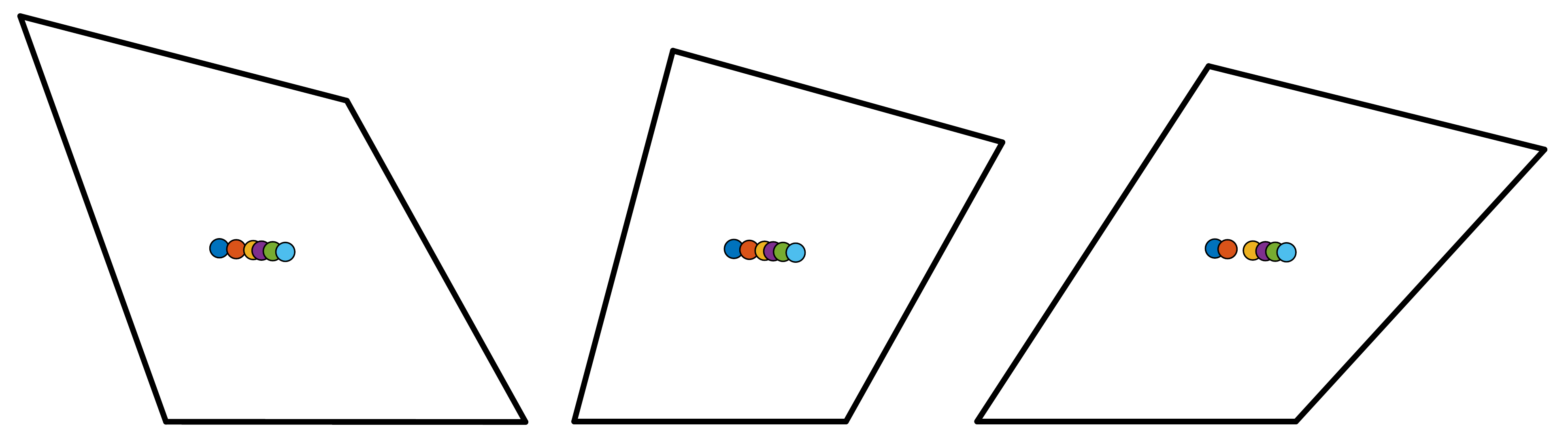}}
\vspace{0.75em}
Times of arrivals of the first-order echoes recorded at marked locations coincide in all three rooms.

Second, we focus on points with affine dimension equal to the ambient dimension and assume that $K \geq 2d$. By enforcing linear dependence among the rows of the matrix $\mathbf{N}_0^T$ (so that it has a non-empty nullspace), we obtain new classes of rooms-trajectory pairs $\cal R$ that give the exact same PPDM. Further analysis shows that the rooms in this class are parallelograms. For example, the following:

\vspace{0.75em}
\centerline{\includegraphics[width=7cm]{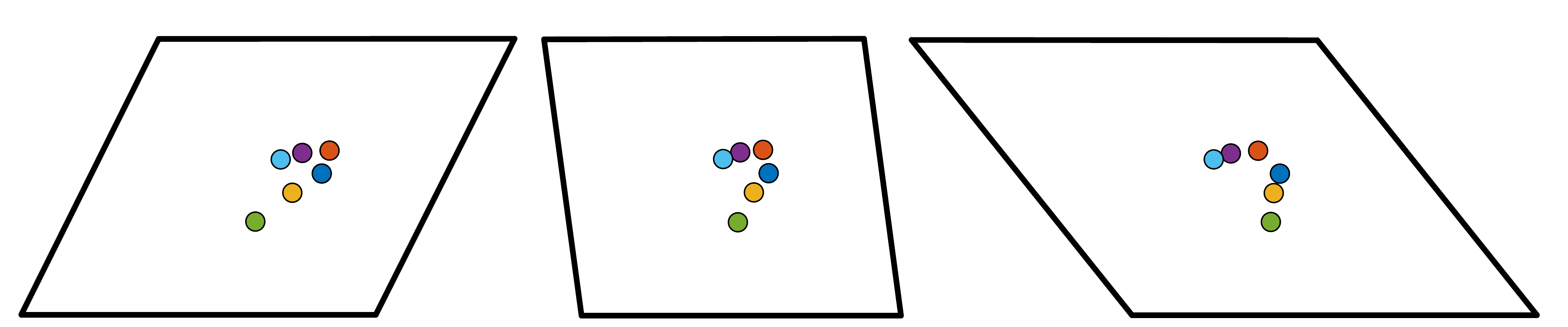}}%
\vspace{0.75em}%
\noindent give the same measurements.

\vspace{-0.5em}
\subsubsection{3D space}
In 3D, we represent the plane normals $\mathbf{n}_j$ and $\mathbf{m}_j$ in spherical coordinates
and obtain three different classes of rooms with equivalent propagation times of first-order echoes.

The first equivalence class of room-trajectory pairs with equal PPDMs is derived analogously to 2D. Points with affine dimension lower than the ambient dimension always result in a non-unique set of distance measurements. In 3D, this comprises collinear and coplanar waypoints.
For example, two rooms from this class  are:

\vspace{1em}
\centerline{\includegraphics[width=7cm]{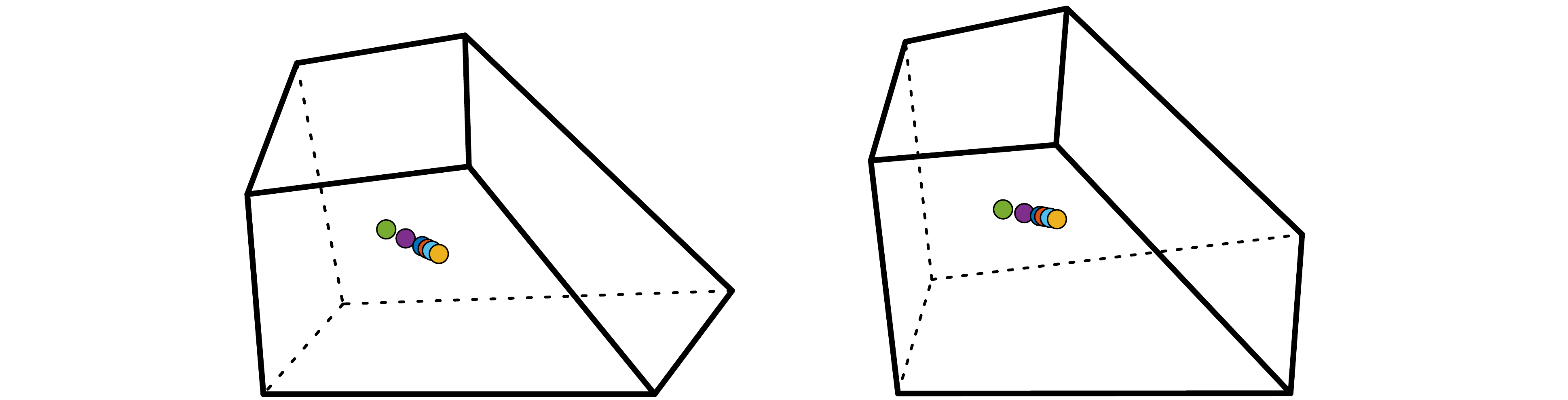}}
\vspace{1em}

The second equivalence class is obtained by considering points with affine dimension equal to the ambient dimension and $K \geq 2d$ (note that if $K < 2d$, $\mN_\mathbf{0}^T$ always has a nullspace). The proof consists in finding a non-singular linear transformation $\mathbf{T} \in \mathbb{R}^{d \times d}$ between the normals, $\vm_j = \mathbf{T} \vn_j$, equivalently, we assume a linear dependence between the columns of the matrix $\mathbf{N}_\mathbf{0}^T$. Importantly, $\mT$ does \emph{not} have to be an orthogonal matrix. Moreover, it has $9$ degrees of freedom for $d=3$, which results in two cases: a simple calculation shows that for $K<9$, we find an infinite number of equivalent rooms for \textit{any} arbitrary room, whereas for $K \geq 9$, ambiguities arise for some particular rooms only, a set of measure zero. An instance for $K = 6$ is illustrated as follows:

\vspace{1em}
\centerline{\includegraphics[width=7cm]{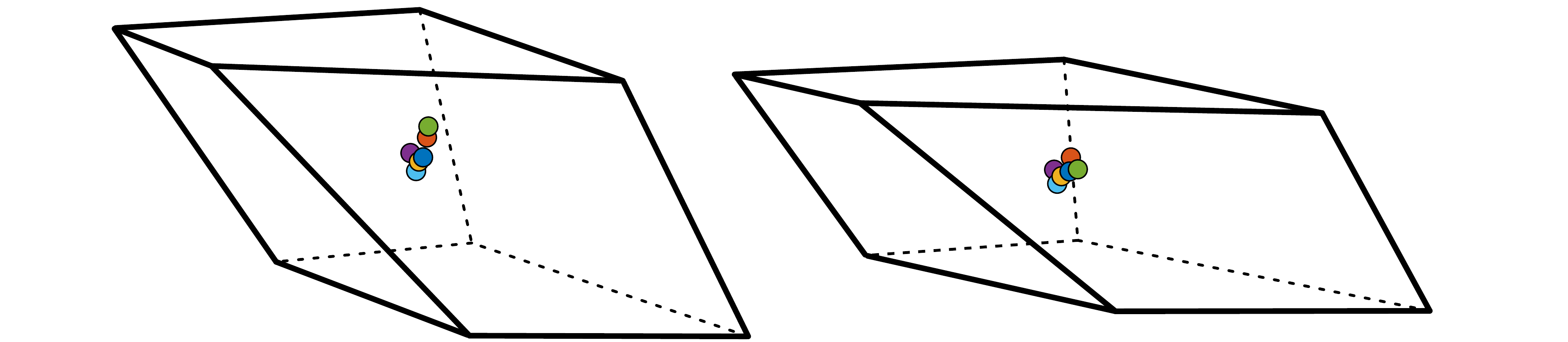}}
\vspace{1em}

The third class is derived by assuming linear dependence among the rows of the matrix $\mathbf{N}_\mathbf{0}^T$. We fix $L < 2d$ linearly independent rows, so that we obtain a fat matrix that always has a nonempty nullspace. Then, we define new rows (i. e. new wall normals $\vn_j$ and $\vm_j$, $L < j \leq K$), as a linear combination of the fixed ones. An example of two rooms from this class is

\vspace{1em}
\centerline{\includegraphics[width=7cm]{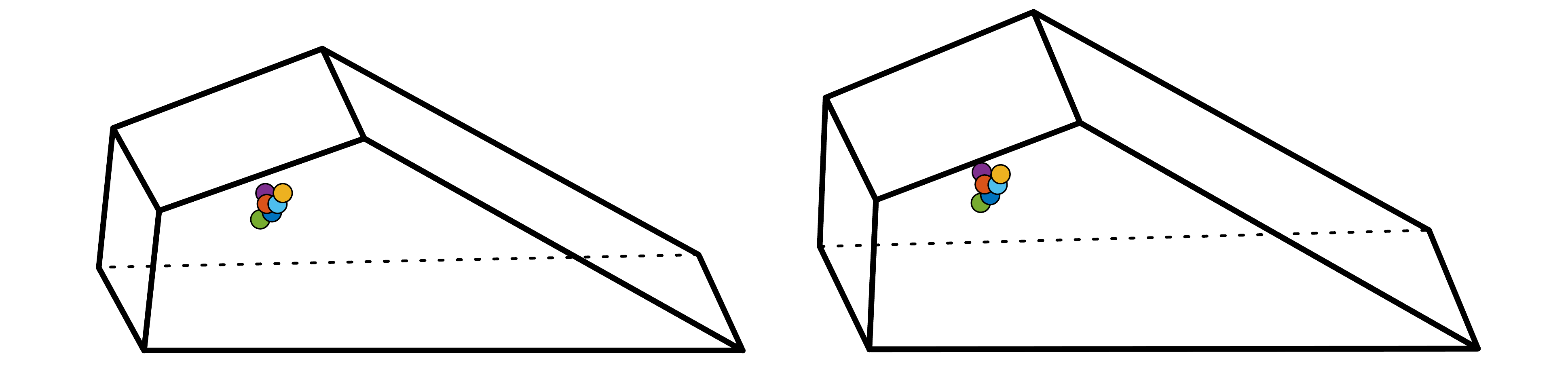}}
\vspace{1em}

To summarize, in this section we explicitly assumed linear dependence between the columns and rows of $\mN_0^T$. It can be shown that the above analysis exhausts all cases when the nullspace of the matrix $\mathbf{N}_\mathbf{0}^T$ is nonempty for $w_j = 0$, $j \leq K$ which lead to valid rooms. Together with Lemma \ref{lemma1}, the above analysis proves

 \begin{figure*}[]
 \vspace{-1em}
 \begin{subfigure}[h]{0.24\linewidth}
  \centering
    \centerline{\includegraphics[width=4cm]{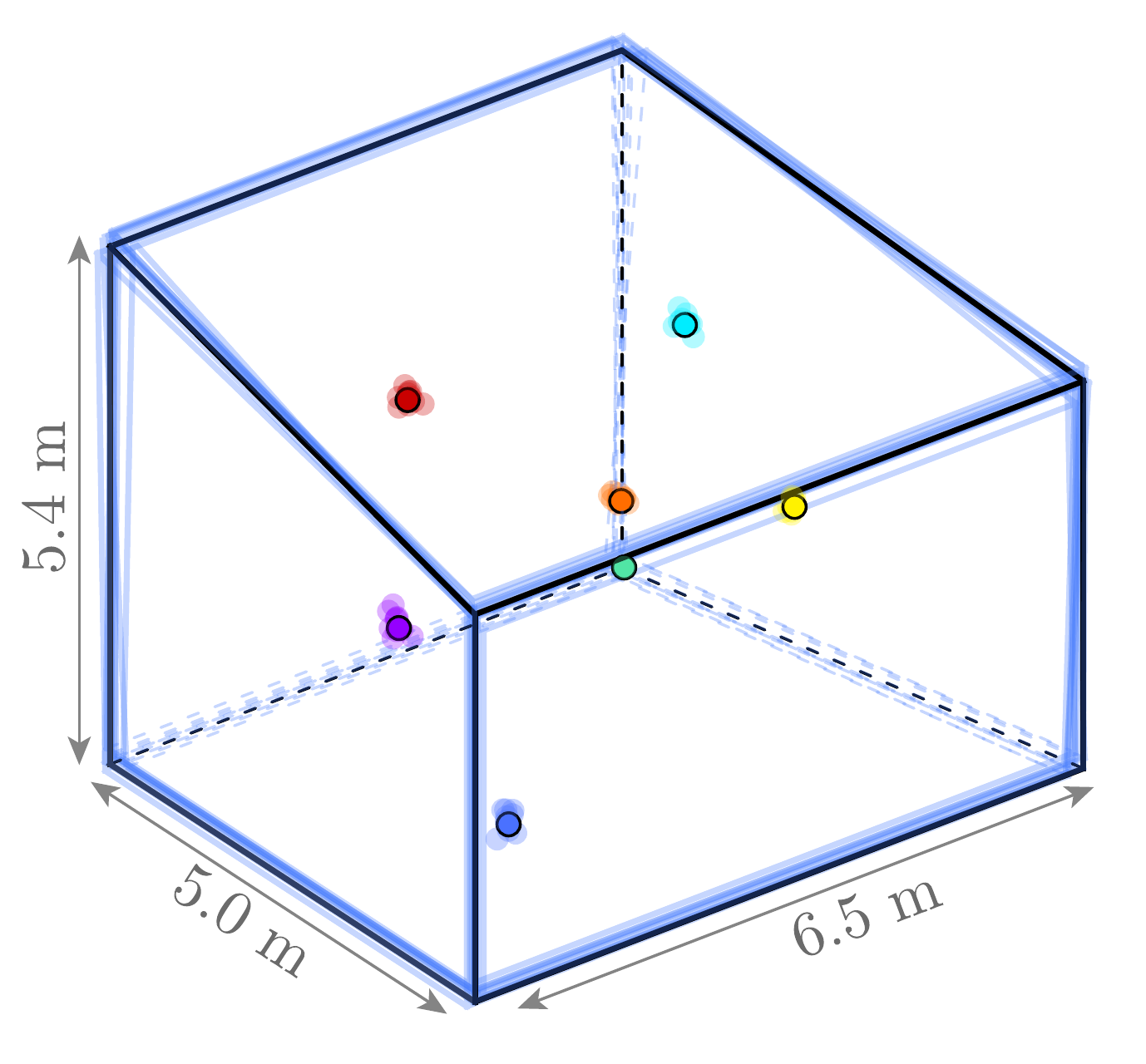}}
    \caption{$\mathcal{N}(0, 0.05^2)$}
 \end{subfigure}
  \begin{subfigure}[h]{0.24\linewidth}
  \centering
    \centerline{\includegraphics[width=4cm]{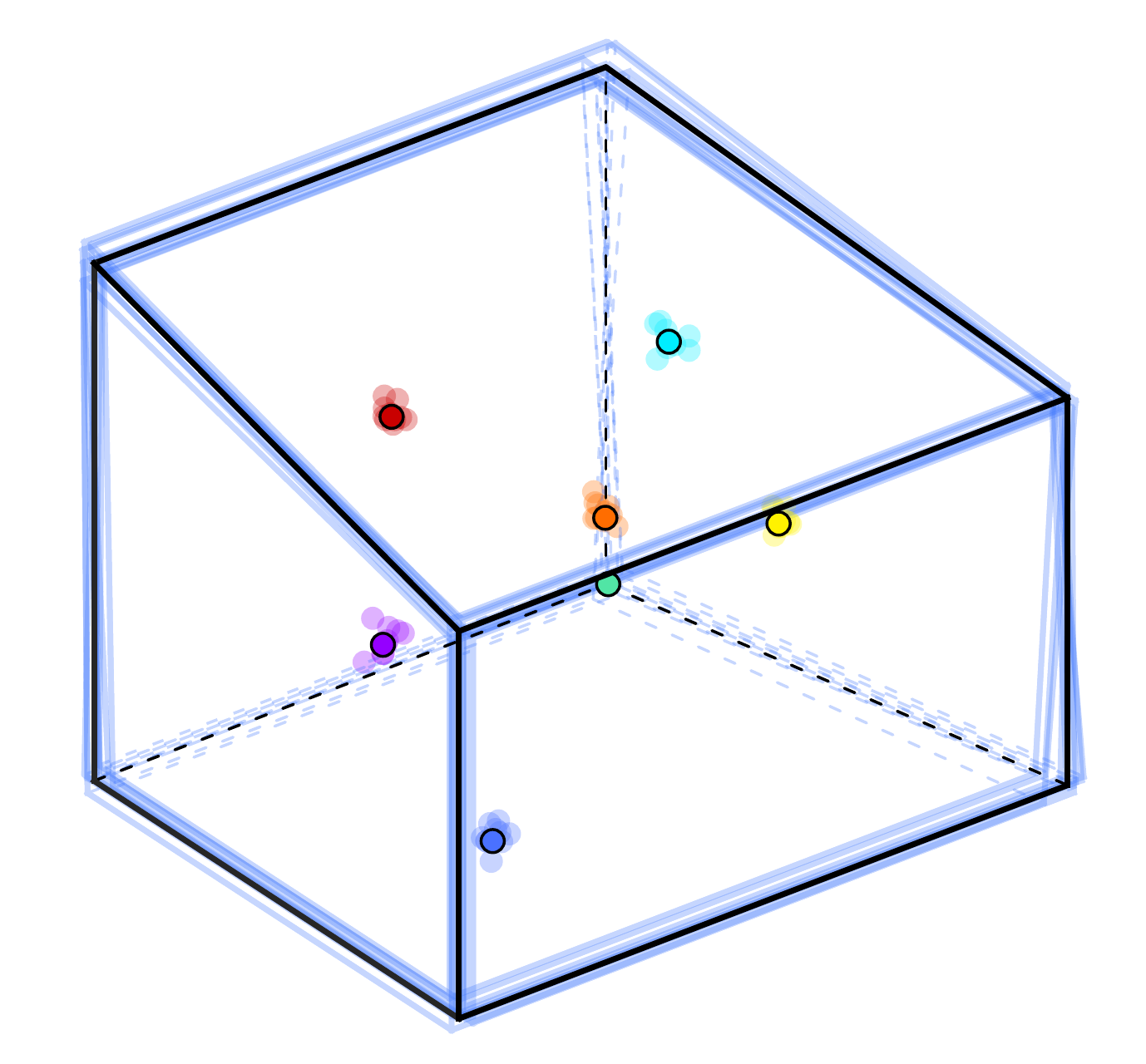}}
    \caption{$\mathcal{N}(0, 0.10^2)$}
 \end{subfigure}
 \begin{subfigure}[h]{0.24\linewidth}
  \centering
    \centerline{\includegraphics[width=4cm]{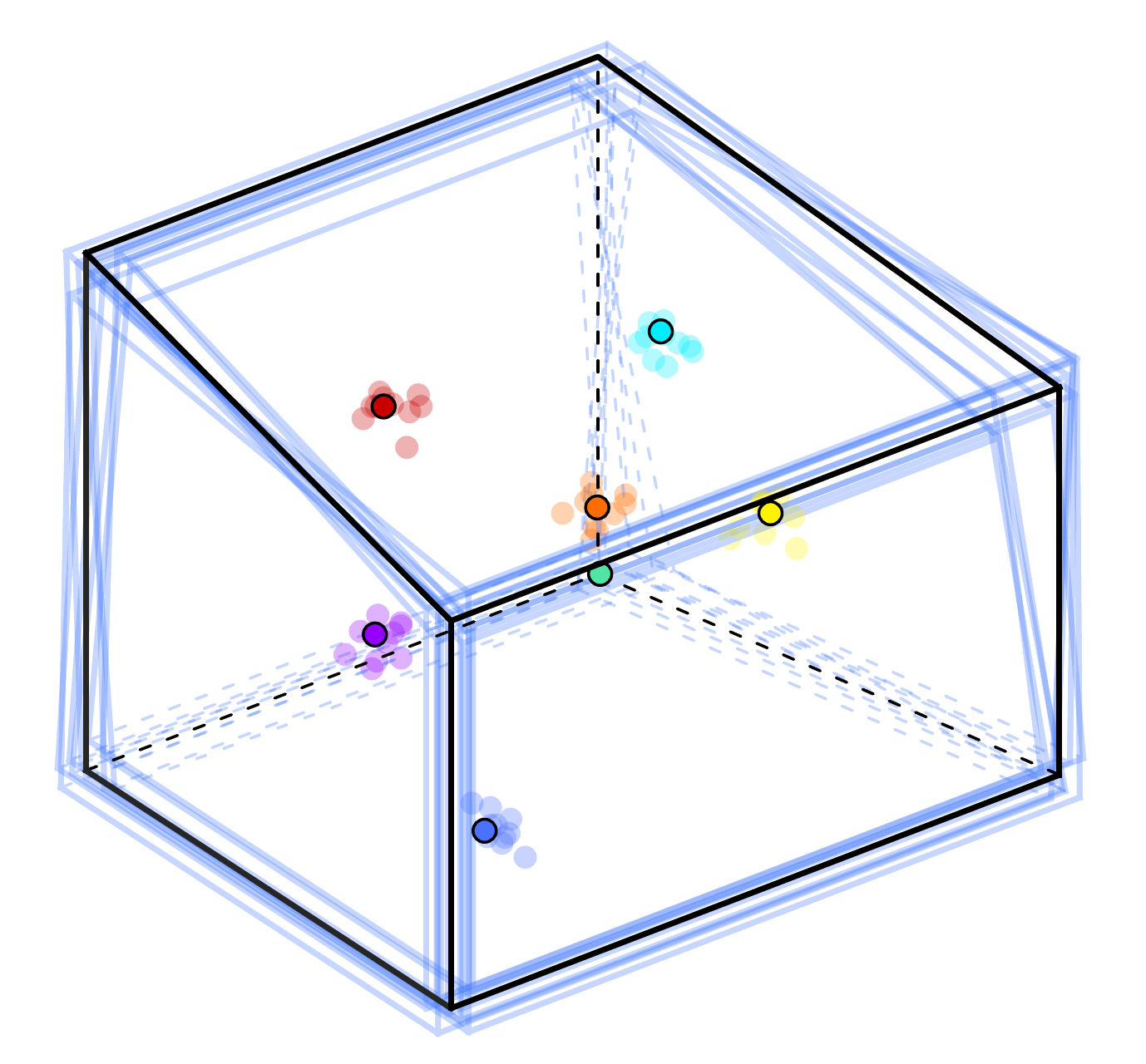}}
    \caption{$\mathcal{N}(0, 0.15^2)$}
 \end{subfigure}
 \begin{subfigure}[h]{0.24\linewidth}
  \centering
  \centerline{\includegraphics[width=4cm]{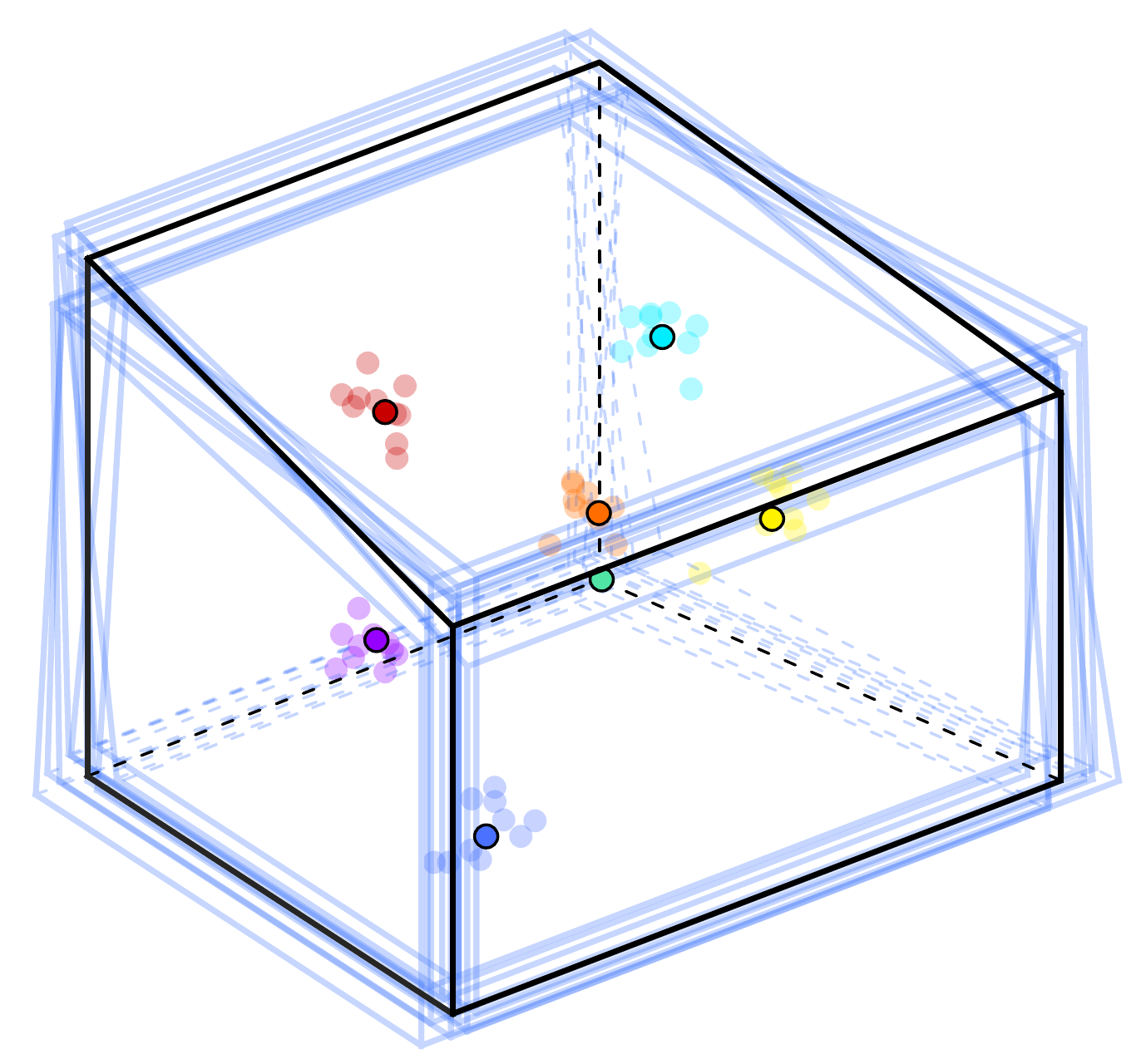}}
  \caption{$\mathcal{N}(0, 0.20^2)$}
 \end{subfigure}
  \vspace{-0.5em}
 \caption{Reconstruction experiments with increasing measurement noise variances.}
 \vspace{-1em}
 \label{fig:noise_room}
\end{figure*}
\begin{theorem}
In 2D, given a room-trajectory pair, we can find another one generating the same PPDM if and only if the waypoints are collinear or the lines enclose a parallelogram. In 3D, given a room-trajectory pair with $K < 9$, we can always find another one generating the same PPDM. For this to happen for $K \geq 9$ the waypoints must be co-planar, or the room must belong to a particular set of measure zero.
\label{thm:first}
\end{theorem}

Therefore, to achieve uniqueness in 3D with $K < 9$, we must add additional information. This could be the distance between several consecutive waypoints or some prior knowledge about the room.

\section{Practical algorithm}

We formulate the joint recovery of points and planes as an
optimization problem, as in our prior work \cite{krekovic}. Noisy measurements are given as
$ \widetilde{d}_{i,j} = d_{i,j} + \epsilon_{i,j}$
where $\epsilon_{i,j}$ is the noise. It is natural to seek the best estimate of the unknown vectors by solving
\begin{align}
\label{eq:cost_function}
& \underset{
\begin{subarray}
\text{ \text{  } \text{  } } q_{j}, \vn_{j}, \vr_{i} \\ i \leq N, j \leq K 
\end{subarray}} {\text{minimize}}
& & \sum_{i = 1}^{N} \sum_{j=1}^{K} (\widetilde{d}_{i,j} - q_{j} + \vn_{j}^{T} \vr_{i})^2 \nonumber \\
& \text{subject to}
& & \norm{\vn_{j}} = 1,\ j = 1, ..., K.
\end{align}
If noise is assumed to be iid normal, then the above program leads to the maximum likelihood estimate.

This cost function is not convex and minimizing it is a priori difficult due to many local minima.
However, different search methods have been developed that guarantee global
convergence in algorithms for nonlinear programming (NLP). 
In particular, the cost function \eqref{eq:cost_function} is suitable for the interior-point filter line-search algorithm for large-scale nonlinear programming (IPOPT) proposed in \cite{wachter1} and implemented in an open-source package as part of COIN-OR Initiative \cite{coinor}. By relying on IPOPT we get a guarantee on global convergence under appropriate (mild) assumptions. Exhaustive computer simulations suggest that the method efficiently (in milliseconds) finds the optimal solution in all test cases.

\section{Numerical simulations}
\label{sec:numerical}

We performed a number of computer simulations to analyze the effect of noise on the success of the reconstruction, both in 2D and 3D. In consideration of space, we only present the 3D reconstructions. To simulate uncertainties in the measurements, we add iid Gaussian noise to the calculated distances and provide them as input to the proposed algorithm.

Figure \ref{fig:noise_room} illustrates joint reconstructions for different variances of Gaussian noise $\mathcal{N}(0, \sigma^2)$, with indicated room dimensions. The original room is colored black, while the original measurement locations are depicted with bordered circles. To generate Figure \ref{fig:noise_room}, we performed 10 experiments for $\sigma = 0.05$, $\sigma = 0.1$, $\sigma = 0.15$ and $\sigma = 0.2$, and overlaid the estimates in lighter shades. To achieve a unique reconstruction, we fixed two normals, i.e. we assumed that the floor and one wall are known, which is a realistic assumption in practice. 
The room is non-shoebox with wall lengths denoted in the figure. 
The distance measurement error with standard deviation $\sigma = 0.2$ [m] is much larger than what is achievable using even simple hardware in a room of such dimensions; nevertheless, the original room-trajectory pair is accurately reconstructed.

Figure \ref{fig:reconstruction} shows the dependency of the estimation errors on the standard deviation of the noise.
The standard deviation is increasing from 0 to 0.22 with steps 0.02 and for each value we performed 5000 experiments. The average SNR is indicated above the graph. 
A room estimation error, defined as the average Euclidean distance between the original and reconstructed room vertices, is plotted in blue (circles).
An estimation error of the measurement locations, defined as the average Euclidean distance between the original and reconstructed measurement locations, is plotted in red (squares). We can see that the reconstruction is stable and the error depends linearly on the noise level.

\begin{figure}[H]
 \vspace{-1em}
  \centering
  \centerline{\includegraphics[width=5.3cm]{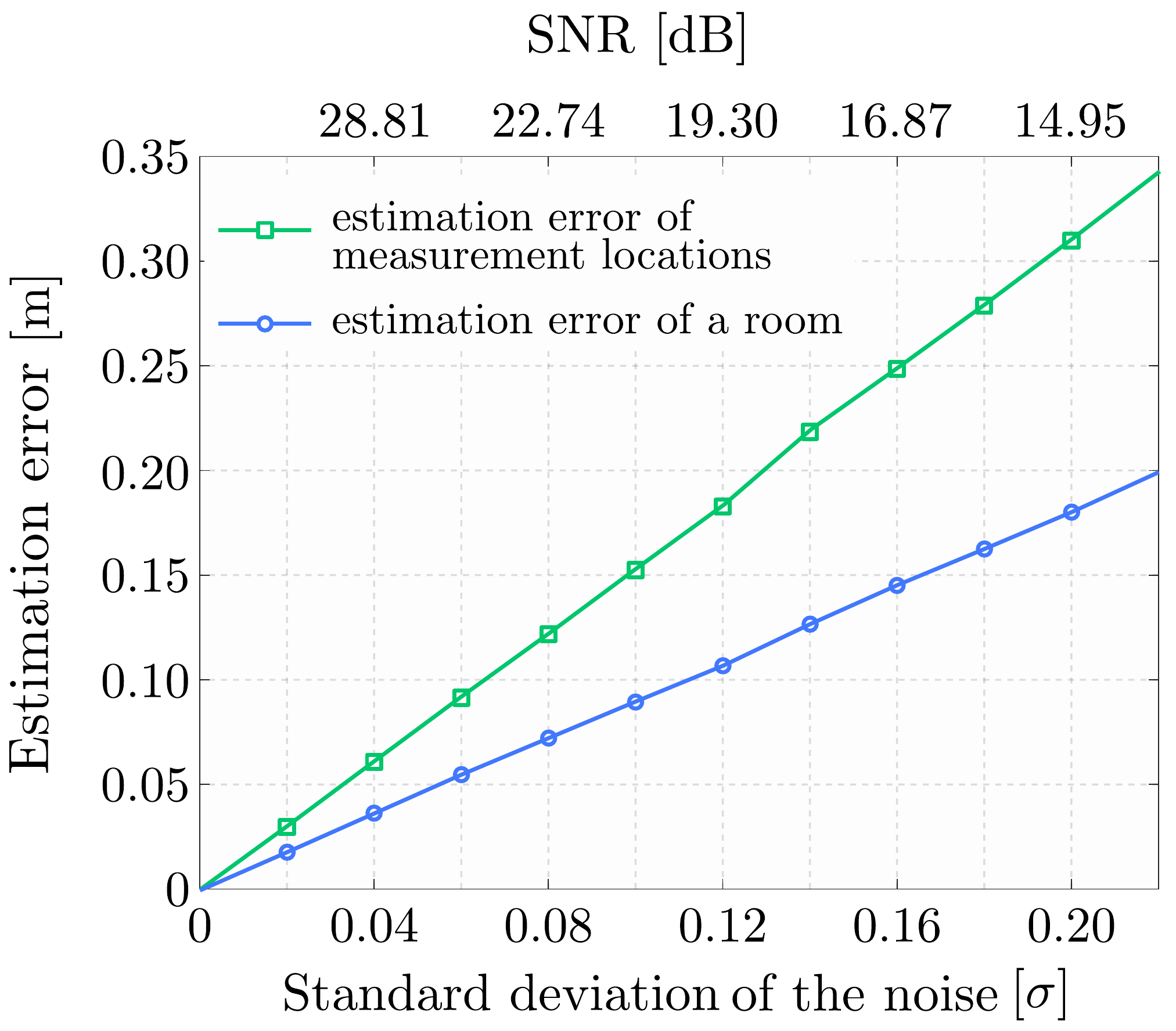}}
 \caption{Dependence of the estimation errors on noise level.}
 \vspace{-1.2em}
 \label{fig:reconstruction}
\end{figure}
 
\section{Conclusion}
\label{sec:conclusion}

We presented an algorithm for reconstructing the 2D and 3D geometry of a room from first-order echoes.
It requires a single device equipped with an omnidirectional microphone and a loudspeaker. 
We noticed that such a setup is an instantiation of a more general inverse problem---reconstruction of the original points and planes from their noisy pairwise distances.
We investigated the uniqueness of this inverse problem and found conditions that guarantee uniqueness. 
We stated our problem as a non-convex optimization problem and proposed a fast optimization tool which simultaneously estimates the planes and points.
Through extensive numerical experiments we showed that our method is robust to noise.
Currently, we are undertaking real experiments, so that the ongoing research includes the verification of our method with real RIRs.

\bibliographystyle{IEEEbib}

\balance
\end{document}